\documentclass[10pt,twocolumn]{article}
\usepackage{geometry}
\usepackage{amsmath}
\usepackage{graphicx}
\usepackage{amssymb}
\usepackage{times}
\usepackage{colortbl}
\usepackage{arydshln}
\usepackage{stfloats}
\usepackage{algorithm}
\usepackage{algorithmicx}
\usepackage{algpseudocode}
\usepackage{multirow}
\usepackage{extarrows}
\usepackage{mathrsfs}
\usepackage{upgreek}
\newtheorem{theorem}{Theorem}[section]
\newtheorem{problem}{Problem}[section]

\newtheorem{assumption}{Assumption}[section]
\newtheorem{definition}{Definition}[section]
\newtheorem{lemma}{Lemma}[section]

\begin{document}
\title{\textbf{Specification-Guided Safety Verification for Feedforward Neural Networks}}

\author{Weiming~Xiang \footnotemark[1],~~Hoang-Dung Tran \footnotemark[1],~~Taylor T. Johnson \footnotemark[1]}

\renewcommand{\thefootnote}{\fnsymbol{footnote}}
	
\maketitle

\footnotetext[1]{Authors are with the Department of Electrical Engineering and Computer Science, Vanderbilt University, Nashville, Tennessee 		37212, USA. Email: xiangwming@gmail.com (Weiming Xiang); Hoang-Dung Tran (trhoangdung@gmail.com); taylor.johnson@gmail.com (Taylor T. Johnson).}

\begin{abstract}
This paper presents a specification-guided safety verification method for feedforward neural networks with general activation functions. As such feedforward networks are memoryless, they can be abstractly represented as mathematical functions, and the reachability analysis of the neural network amounts to interval analysis problems. In the framework of interval analysis, a computationally efficient formula which can quickly compute the output interval sets of a neural network is developed. Then, a specification-guided reachability algorithm is developed. Specifically, the bisection process in the verification algorithm is completely guided by a given safety specification. Due to the employment of the safety specification, unnecessary computations are avoided and thus computational cost can be reduced significantly. Experiments show that the proposed method enjoys much more efficiency in safety verification with significantly less computational cost.  
\end{abstract}

\maketitle

\section{Introduction}
Artificial neural networks have been widely used in machine learning systems. Though neural networks have been showing effectiveness and powerful ability in resolving complex problems, they are confined to systems which comply only to the lowest safety integrity levels since, in most of time, a neural network is viewed as a \emph{black box} without effective  methods to assure safety specifications for its outputs. Neural networks are trained over a finite number of input and output data, and are expected to be able to generalize to produce desirable outputs for given inputs even including previously unseen inputs. However, in many practical applications, the number of inputs is essentially infinite, this means it is impossible to check all the possible inputs only by performing experiments and moreover, it has been observed that neural networks can react in unexpected and incorrect ways to even slight perturbations of their inputs \cite{szegedy2013intriguing}, which could result in unsafe systems. Hence, methods that are able to provide formal guarantees are in a great demand for verifying specifications or properties of neural networks.  Verifying neural networks is a hard problem, even simple properties about them have been proven NP-complete
problems \cite{katz2017reluplex}. The difficulties mainly come from the presence of activation functions and the complex structures, making neural networks large-scale, nonlinear, non-convex and thus incomprehensible to humans. 

The importance of methods of formal guarantees for neural networks has been well-recognized in literature. There exist a number of results for verification of feedforward neural networks, especially for Rectifier Linear Unit (ReLU) neural networks, and a few results are devoted to neural networks with broad classes of activation functions. Motivated to general class of neural networks such as those considered in \cite{xiang2018output}, our key contribution in this paper is to develop a specification-guided method for safety verification of feedforward neural network. First, we formulate the safety verification problem in the framework of interval arithmetic, and provide a computationally efficient formula to compute output interval sets. The developed formula is able to calculate the output intervals in a fast manner. Then, analogous to other state-of-the-art verification methods, such as counterexample-guided abstraction refinement (CEGAR) \cite{clarke2000counterexample} and property directed reachability (PDR) \cite{een2011efficient}, and inspired by the Moore-Skelboe algorithm \cite{skelboe1974computation}, a specification-guided algorithm is developed. Briefly speaking, the safety specification is utilized to examine the existence of intersections between output intervals and unsafe regions and then determine the bisection actions in the verification algorithm. By making use of the information of safety specification, the computation cost can be reduced significantly. We provide experimental evidences to show the advantages of specification-guided approach, which shows that our approach only needs about 3\%--7\% computational cost of the method proposed in \cite{xiang2018output} to solve the same safety verification problem.

\section{Related Work}
Many recent works are focusing on ReLU neural networks. 
In \cite{katz2017reluplex}, an SMT solver named Reluplex is proposed for a special class of neural networks with ReLU activation functions. The Reluplex extends the well-known Simplex algorithm from linear functions to ReLU functions by making use of the piecewise linear feature of ReLU functions.  In \cite{xiang2017reachable_arxiv}, A layer-by-layer approach is developed for the output reachable set computation of ReLU neural networks. The computation is formulated in the form of a set of manipulations for a union of polyhedra. A verification engine for ReLU neural networks called $\mathrm{AI}^2$ was proposed in \cite{gehr2018ai}. In their approach, the authors abstract perturbed inputs and safety specifications as zonotopes, and reason about their behavior using operations for zonotopes. An Linear Programming (LP)-based method is proposed  \cite{ehlers2017formal}, and in \cite{Lomuscio2017an_arxiv} authors encoded the constraints of ReLU functions as a Mixed-Integer Linear Programming (MILP). Combining output specifications that are expressed in terms of LP, the verification problem for output set eventually turns to a feasibility problem of MILP. In \cite{dutta2017output,dutta2018output}, an MILP based verification engine called Sherlock that performs an output range analysis of ReLU feedforward neural networks is proposed, in which a combined local and global search is developed to more efficiently solve MILP. 

Besides the results for ReLU neural networks, there are a few other results for neural networks with general activation functions. In \cite{pulina2010abstraction,pulina2012challenging}, a piecewise-linearization of the nonlinear activation functions is used to reason about their behaviors. In this framework, the authors replace the activation functions with piecewise constant approximations and use the bounded model checker hybrid satisfiability (HySAT) \cite{franzle2007hysat} to analyze various properties. In their papers, the authors highlight the difficulty of scaling this technique and, currently, are only able to tackle small networks with at most 20 hidden nodes. In \cite{huang2017safety}, the authors proposed a framework for verifying the safety of network image classification decisions by searching for adversarial examples within a specified region.
A adaptive nested optimization framework is proposed for reachability problem of neural networks in \cite{ruan2018reachability}.  
In \cite{xiang2018output}, a simulation-based approach was developed, which used a finite number of simulations/computations to estimate the reachable set of multi-layer neural networks in a general form. Despite this success, the approach lacks the ability to resolve the reachable set computation problem for neural networks that are large-scale, non-convex, and nonlinear. Still, simulation-based approaches, like the one developed in \cite{xiang2018output}, present a plausibly practical and efficient way of reasoning about neural network behaviors. The critical step in improving simulation-based approaches is bridging the gap between finitely many simulations and the essentially infinite number of inputs that exist in the continuity set. Sometimes, the simulation-based approach requires a large number of simulations to obtain a tight reachable set estimation, which is computationally costly in practice. In this paper, our aim is to reduce the computational cost by avoiding unnecessary computations with the aid of a specification-guided method.

\section{Background}
\subsection{Feedforward Neural Networks}
Generally speaking, a neural network consists of a number of interconnected neurons and each neuron is a simple processing element that responds to the weighted inputs it received from other neurons. In this paper, we consider feed-forward neural networks, which generally consist of one input layer, multiple hidden layers and one output layer.
The action of a neuron depends on its activation function, which is in the form of
\begin{align}
	y_i = \phi\left(\sum\nolimits_{j=1}^{n}\omega_{ij} x_j + \theta_i\right)
\end{align}
where $x_j$ is the $j$th input of the $i$th neuron, $\omega_{ij}$ is the weight from the $j$th input to the $i$th neuron, $\theta_i$ is called the bias of the $i$th neuron, $y_i$ is the output of the $i$th neuron, $\phi(\cdot)$ is the activation function. The activation function is generally a nonlinear continuous function  describing the reaction of $i$th neuron with inputs $x_j$, $j=1,\cdots,n$. Typical activation functions include ReLU, logistic, tanh, exponential linear unit, linear functions, for instance. In this work, our approach aims at being capable of dealing with activation functions regardless of their specific forms.

A feedforward neural network has multiple layers,  and each layer $\ell$, $1 \le \ell \le L $, has $n^{\{\ell\}}$ neurons.  In particular, layer $\ell =0$ is used to denote the input layer and $n^{\{0\}}$ stands for the number of inputs in the rest of this paper. For the layer $\ell$, the corresponding input vector is denoted by $\mathbf{x}^{\{\ell\}}$ and the weight matrix is 
\begin{equation}
\mathbf{W}^{\{\ell\}} = \left[\omega_{1}^{\{\ell\}},\ldots,\omega_{n^{\{\ell\}}}^{\{\ell\}}\right]^{\top}
\end{equation}
where $\omega_{i}^{\{\ell\}}$ is the weight vector. The bias vector for layer $\ell$ is
\begin{equation} 
\boldsymbol {\uptheta}^{\{\ell\}}=\left[\theta_1^{\{\ell\}},\ldots,\theta_{n^{\{\ell\}}}^{\{\ell\}}\right]^{\top}.
\end{equation} 

The output vector of layer $\ell$ can be expressed as 
\begin{equation}
\mathbf{y}^{\{\ell\}}=\phi_{\ell}(\mathbf{W}^{\{\ell\}}\mathbf{x}^{\{\ell\}}+\uptheta^{\{\ell\}})
\end{equation} 
where $\phi_{\ell}(\cdot)$ is the activation function of layer $\ell$.

The output of $\ell-1$ layer is the input of $\ell$ layer, and the mapping from the input of input layer, that is $\mathbf{x}^{[0]}$, to the output of output layer, namely $\mathbf{y}^{[L]}$, stands for the input-output relation of the neural network, denoted by
\begin{equation}\label{NN}
\mathbf{y}^{\{L\}} = \Phi (\mathbf{x}^{\{0\}})
\end{equation}    
where $\Phi(\cdot) \triangleq \phi_L  \circ \phi_{L - 1}  \circ  \cdots  \circ \phi_1(\cdot) $.

\subsection{Problem Formulation}
We start by defining the neural network output set that will become of interest all through the rest of this paper.
\begin{definition}
	Given a feedforward neural network in the form of (\ref{NN}) and an input
	set $\mathcal{X} \subseteq  \mathbb{R}^{n^{\{0\}}}$, the following set
	\begin{align}
		\mathcal{Y} = \left\{\mathbf{y} ^{\{L\}} \in \mathbb{R}^{n^{\{L\}}} \mid \mathbf{y}^{\{L\}} = \Phi (\mathbf{x}^{\{0\}}),~ \mathbf{x}^{\{0\}} \in \mathcal{X}\right\} \label{output_set}
	\end{align}
	is called the output set of neural network (\ref{NN}). 
\end{definition}

The safety specification of a neural network is expressed by a set defined in the output space, describing the safety requirement.
\begin{definition}
	Safety specification $\mathcal{S}$ formalizes the safety requirements for output $\mathbf{y}^{[L]}$ of neural network (\ref{NN}), and is a predicate over output $\mathbf{y}^{[L]}$ of neural network  (\ref{NN}). The neural network  (\ref{NN}) is safe if and only if the	following condition is satisfied:
	\begin{equation}\label{verification}
	\mathcal{Y} \cap \neg \mathcal{S} = \emptyset
	\end{equation}
	where $\mathcal{Y}$ is the output set defined by (\ref{output_set}), and $\neg$ is the symbol for logical negation.
\end{definition}

The safety verification problem for the neural network  (\ref{NN}) is stated as follows.

\begin{problem}\label{problem}
	How does one verify the safety requirement described by (\ref{verification}), given a neural network (\ref{NN}) with a compact input set $\mathcal{X}$ and a
	safety specification $\mathcal{S}$?
\end{problem}

The key for solving the safety verification Problem \ref{problem} is computing output set $\mathcal{Y}$. However, since neural networks are often nonlinear and non-convex, it is extremely difficult to compute the exact output set $\mathcal{Y}$. Rather than directly computing the exact output set for a neural network, a more practical and feasible way for safety verification is to derive an over-approximation of $\mathcal{Y}$.

\begin{definition}\label{def2}
	A set $\mathcal{Y}_o$ is an over-approximation of $\mathcal{Y}$ if $\mathcal{Y} \subseteq \mathcal{Y}_o$ holds.  
\end{definition}

The following lemma implies that it is sufficient to use the over-approximated output set for the safety verification of a neural network.

\begin{lemma}\label{lemma1}
	Consider a neural network in the form of (\ref{NN}) and a safety specification $\mathcal{S}$, the neural network is safe if the following condition is satisfied
	\begin{equation}\label{lemma1_1}
	\mathcal{Y}_o \cap \neg \mathcal{S} = \emptyset
	\end{equation}
	where $\mathcal{Y} \subseteq\mathcal{Y}_o$.
\end{lemma}
\begin{proof}
	Due to $\mathcal{Y} \subseteq\mathcal{Y}_o$,  (\ref{lemma1_1}) implies $\mathcal{Y} \cap\neg \mathcal{S} = \emptyset$.
\end{proof}

From Lemma \ref{lemma1}, the problem turns to how to construct an appropriate over-approximation $\mathcal{Y}_o$. One natural way, as the method developed in \cite{xiang2018output}, is to find a set $\mathcal{Y}_o$ as small as possible to tightly over-approximate output set $\mathcal{Y}$ and further perform safety verification. However, this idea sometimes could be computationally expensive, and actually most of computations are unnecessary for safety verification. In the following, a specification-guided approach will be developed, and the over-approximation of output set is computed in an adaptive way with respect to a given safety specification.

\section{Safety Verification}
\subsection{Preliminaries and Notation}
Let $[x] = [\underline{x}, \overline{x}]$, $[y] = [\underline{y},\overline{y}]$ be real compact intervals and $\circ$ be one of the basic operations addition,
subtraction, multiplication and division, respectively, for real numbers, that is $\circ \in \{+,-,\cdot, / \}$, where it is assumed that $0 \notin [b]$ in case of division. We define these operations for intervals $[x]$ and $[y]$ by $[x] \circ [y] = \{x \circ y \mid x \in [y],x\in [y]\}$. The width of an interval $[x]$ is defined and denoted by $w([x]) = \overline{x} - \underline{x}$. The set of compact intervals in $\mathbb{R}$ is denoted by $\mathbb{IR}$. We say $[\phi]: \mathbb{IR} \to \mathbb{IR}$ is an interval extension of function $\phi: \mathbb{R} \to \mathbb{R}$, if for any degenerate interval arguments, $[\phi]$ agrees with $\phi$ such that $[\phi]([x,x]) = \phi(x)$. In order to consider multidimensional problems where $\mathbf{x} \in \mathbb{R}^{n}$ is taken into account, we denote $[\mathbf{x}] =[\underline{x}_1,\overline{x}_1]\times\cdots \times[\underline{x}_n,\overline{x}_n] \in \mathbb{IR}^{n}$, where $\mathbb{IR}^n$ denotes the set of compact interval in $\mathbb{R}^n$. The width of an interval vector $\mathbf{x}$ is the largest of the widths of any
of its component intervals $w([\mathbf{x}])= \max_{i=1,\ldots,n} (\overline{x}_i-\underline{x}_i)$. A mapping $[\Phi] : \mathbb{IR}^{n} \to \mathbb{IR}^{m}$ denotes the interval extension of a  function $\Phi:\mathbb{R}^{n} \to  \mathbb{R}^m$. An interval extension is inclusion monotonic if, for any $[\mathbf{x}_1],[\mathbf{x}_2] \in \mathbb{IR}^{n}$, $[\mathbf{x}_1] \subseteq [\mathbf{x}_2]$ implies $[\Phi]([\mathbf{x}_1]) \subseteq [\Phi]([\mathbf{x}_2])$. A fundamental property of inclusion monotonic interval extensions is that $\mathbf{x} \in [\mathbf{x}] \Rightarrow \Phi(\mathbf{x}) \in [\Phi]([\mathbf{x}])$, which means the value of $\Phi$ is contained in the interval $[\Phi]([\mathbf{x}])$ for every $\mathbf{x}$ in $[\mathbf{x}]$.

Several useful definitions and lemmas are presented.

\begin{definition} \cite{moore2009introduction}
	Piece-wise monotone functions, including exponential, logarithm, rational power, absolute value, and trigonometric functions, constitute the set of standard functions.
\end{definition}

\begin{lemma} \label{lemma2}\cite{moore2009introduction}
	A function $\Phi$ which is composed by finitely many elementary operations $\{+,-,\cdot, / \}$ and standard functions is inclusion monotone.
\end{lemma}

\begin{definition} \cite{moore2009introduction}
	An interval extension $[\Phi]([\mathbf{x}])$ is said to be Lipschitz in $[\mathbf{x}_0]$ if there is a constant $\xi$
	such that $w([\Phi]([\mathbf{x}]))\le \xi w([\mathbf{x}])$ for every $[\mathbf{x}] \subseteq [\mathbf{x}_0]$.
\end{definition}

\begin{lemma}\label{lemma3}\cite{moore2009introduction}
	If a function $\Phi(\mathbf{x})$ satisfies an ordinary Lipschitz condition in $[\mathbf{x_0}]$,
	\begin{equation}
	\left\|\Phi(\mathbf{x}_2)-\Phi(\mathbf{x}_1)\right\| \le \xi\left\|\mathbf{x}_2-\mathbf{x}_1\right\|,~\mathbf{x}_1,\mathbf{x}_2 \in [\mathbf{x}_0]
	\end{equation}
	then the interval extension $[\Phi]([\mathbf{x}])$ is a Lipschitz interval extension in $[\mathbf{x}_0]$, 
	\begin{equation}
	w([\Phi]([\mathbf{x}]))\le \xi w([\mathbf{x}]),~[\mathbf{x}] \subseteq [\mathbf{x_0}].
	\end{equation}
\end{lemma}

The following trivial assumption is given for activation functions. 

\begin{assumption}\label{assumption_0}
	The activation function $\phi$ considered in this paper is composed by finitely many elementary operations and standard functions. 
\end{assumption}

Based on Assumption \ref{assumption_0}, the following result can be obtained for a feedforward neural network. 

\begin{theorem}\label{thm1}
	The interval extension $[\Phi ]$ of neural network $\Phi$ composed by activation functions satisfying Assumption \ref{assumption_0} is inclusion monotonic and Lipschitz such that 
	\begin{equation}\label{L_NN}
	w([\Phi]([\mathbf{x}]))\le \xi^{L}\prod\nolimits_{\ell = 1}^L {\left\| {\mathbf{W}^{\{ \ell\} } } \right\|}  w([\mathbf{x}]),~[\mathbf{x}] \subseteq \mathbb{IR}^{n^{\{0\}}}
	\end{equation}
	where $\xi$ is a Lipschitz constant for all activation functions in $\Phi$. 
\end{theorem}
\begin{proof}
	Under Assumption \ref{assumption_0}, the inclusion monotonicity can be obtained directly based on Lemma \ref{lemma2}. Then, for the layer $\ell$, we denote $\hat\phi_{\ell}(\mathbf{x}^{\{\ell\}}) = \phi_{\ell} (\mathbf{W}^{\{\ell\}} \mathbf{x}^{\{\ell\}}  + \boldsymbol{\uptheta}^{\{\ell\}}  )$. For any $\mathbf{x}_1,\mathbf{x}_2$, it has
	\begin{align*}
		\left\| {\hat \phi _{\ell} (\mathbf{x}_2^{\{ \ell\} } ) - \hat \phi _{\ell} (\mathbf{x}_1^{\{ \ell\} } )} \right\| \leq \xi \left\| {\mathbf{W}^{\{ \ell\} } \mathbf{x}_2^{\{ \ell\} }  - \mathbf{Wx}_1^{\{ \ell\} } } \right\| \nonumber
		\\
		\leq \xi \left\| {\mathbf{W}^{\{ \ell\} } } \right\|\left\| {\mathbf{x}_2^{\{ \ell\} }  - \mathbf{x}_1^{\{ \ell\} } } \right\|.
	\end{align*}
	
	Due to $\mathbf{x}^{\{\ell\}}=\hat\phi_{\ell-1}(\mathbf{x}^{\{\ell-1\}})$, $\ell=1,\ldots,L$, we have $\xi^{L}\prod\nolimits_{\ell = 1}^L {\left\| {\mathbf{W}^{\{ \ell\} } } \right\|}$ the Lipschitz constant for $\Phi$, and (\ref{L_NN}) can be established by Lemma \ref{lemma3}. 
\end{proof}

\subsection{Interval Analysis}

First, we consider a single layer $\mathbf{y} = \phi(\mathbf{W}\mathbf{x}+\boldsymbol{\uptheta})
$. Given an interval input $[\mathbf{x}]$, the interval extension is $[\phi](\mathbf{W}[\mathbf{x}]+\boldsymbol{\uptheta}) = [\underline{y}_1,\overline{y}_1]\times\cdots\times[\underline{y}_n,\overline{y}_n] = [\mathbf{y}]$, where 
\begin{align}
	\underline{y}_i &= \min_{\mathbf{x} \in [\mathbf{x}]} \phi\left(\sum\nolimits_{j=1}^{n}\omega_{ij} x_j + \theta_i\right)\label{thm1_1}
	\\
	\overline y_i &= \max_{\mathbf{x} \in [\mathbf{x}]} \phi\left(\sum\nolimits_{j=1}^{n}\omega_{ij} x_j + \theta_i\right) . \label{thm1_2}
\end{align} 

To compute the interval extension $[\phi]$, we need to compute the minimum and maximum values of the output of nonlinear function $\phi$. For general nonlinear functions, the optimization problems are still challenging. Typical activation functions include ReLU, logistic, tanh, exponential linear unit, linear functions, for instance, satisfy the following monotonic assumption.

\begin{assumption}\label{assumption_1}
	For any two scalars $z_1 \le z_2$, the activation function satisfies $\phi(z_1) \le \phi(z_2)$. 
\end{assumption}

Assumption \ref{assumption_1} is a common property that can be satisfied by a variety of activation functions. For example, it is easy to verify that the most commonly used such as logistic, tanh,  ReLU, all satisfy Assumption \ref{assumption_1}. Taking advantage of the monotonic property of $\phi$, the interval extension $[\phi]([z]) = [\phi(\underline{z}),\phi(\overline{z})]$. Therefore, $\underline{y}_i$ and $\overline{y}_i$ in (\ref{thm1_1}) and (\ref{thm1_2}) can be explicitly written out  as
\begin{align} \label{y_1}
	\underline{y}_i & = \sum\nolimits_{j=1}^{n}\underline{p}_{ij} +   \theta_i
	\\
	\overline{y}_i &= \sum\nolimits_{j=1}^{n}\overline{p}_{ij} +   \theta_i \label{y_2}
\end{align}
with $\underline{p}_{ij}$ and $\overline{p}_{ij}$ defined by
\begin{align} \label{y_3}
	\underline{p}_{ij}  &= \left\{ {\begin{array}{*{20}l}
			{\omega _{ij} \underline{x}_j,} & {\omega _{ij}\geq 0}  \\
			{\omega _{ij} \overline x_j ,} & {\omega _{ij}  < 0}  \\
		\end{array} } \right.
		\\
		\overline p_{ij}& = \left\{ {\begin{array}{*{20}c}
				{\omega _{ij} \overline x_j ,} & {\omega _{ij}  \geq 0}  \\
				{\omega _{ij} \underline{x}_j ,} & {\omega _{ij}  < 0}  \\
			\end{array} } \right.. \label{y_4}
		\end{align}
		
		From (\ref{y_1})--(\ref{y_4}), the output interval of a single layer can be efficiently computed with these explicit expressions. Then, we consider the feedforward neural network $\mathbf{y}^{\{L\}}=\Phi(\mathbf{x}^{\{0\}})$ with multiple layers, the interval extension $[\Phi ]([\mathbf{x}^{\{ 0\} } ])$ can be computed by the following layer-by-layer computation. 
		\begin{theorem}\label{thm2}
			Consider feedforward neural network (\ref{NN}) with activation function satisfying Assumption \ref{assumption_1} and an interval input $[\mathbf{x}^{\{0\}}]$, an interval extension can be determined by
			\begin{equation} \label{thm2_1}
			[\Phi ]([\mathbf{x}^{\{ 0\} } ]) = [\hat \phi _L ] \circ  \cdots  \circ [\hat \phi _1 ] \circ [\hat \phi _0 ]([\mathbf{x}^{\{ 0\} } ])
			\end{equation}
			where $[\hat \phi_{\ell}]([\mathbf{x}^{\{\ell\}}]) =[\phi_{\ell} ](\mathbf{W}^{\{\ell\}} [\mathbf{x}^{\{\ell\}} ] + \boldsymbol{\uptheta}^{\{\ell\}}  )=[\mathbf{y}^{\{\ell\}}]$ in which
			\begin{align} \label{thm2_2}
				\underline{y}_i^{\{\ell\}} & = \sum\nolimits_{j=1}^{n^{\{\ell\}}}\underline{p}_{ij}^{\{\ell\}} +   \theta_i^{\{\ell\}}
				\\
				\overline{y}_i^{\{\ell\}} &= \sum\nolimits_{j=1}^{n^{\{\ell\}}}\overline{p}_{ij}^{\{\ell\}} +   \theta_i^{\{\ell\}} \label{thm2_3}
			\end{align}
			with $\underline{p}_{ij}^{\{\ell\}}$ and $\overline{p}_{ij}^{\{\ell\}}$ defined by
			\begin{align} \label{thm2_4}
				\underline{p}_{ij}^{\{\ell\}}  &= \left\{ {\begin{array}{*{20}l}
						{\omega _{ij}^{\{\ell\}} \underline{x}_j^{\{\ell\}},} & {\omega _{ij}^{\{\ell\}}\geq 0}  \\
						{\omega _{ij}^{\{\ell\}} \overline x_j^{\{\ell\}} ,} & {\omega _{ij}^{\{\ell\}}  < 0}  \\
					\end{array} } \right.
					\\
					\overline p_{ij}^{\{\ell\}}& = \left\{ {\begin{array}{*{20}c}
							{\omega _{ij}^{\{\ell\}} \overline x_j^{\{\ell\}} ,} & {\omega _{ij}^{\{\ell\}}  \geq 0}  \\
							{\omega _{ij}^{\{\ell\}} \underline{x}_j^{\{\ell\}} ,} & {\omega _{ij}^{\{\ell\}}  < 0}  \\
						\end{array} } \right.. \label{thm2_5}
					\end{align}
				\end{theorem}
				\begin{proof}
					We denote $\hat\phi_{\ell}(\mathbf{x}^{\{\ell\}}) = \phi_{\ell} (\mathbf{W}^{\{\ell\}} \mathbf{x}^{\{\ell\}}  + \boldsymbol{\uptheta}^{\{\ell\}}  )$. For a feedforward neural network, it essentially has $\mathbf{x}^{\{\ell\}}=\hat\phi_{\ell-1}(\mathbf{x}^{\{\ell-1\}})$, $\ell=1,\ldots,L$ which leads to (\ref{thm2_1}). Then, for each layer, the interval extension $[\mathbf{y}^{\{\ell\}}]$ computed by (\ref{thm2_2})--(\ref{thm2_5}) can be obtained directly from (\ref{y_1})--(\ref{y_4}). 
				\end{proof}
				
				We denote the set image for neural network $\Phi$ as follows
				\begin{equation}
				\Phi([\mathbf{x}^{\{0\}}])=\{\Phi(\mathbf{x}^{\{0\}}):\mathbf{x}^{\{0\}} \in [\mathbf{x}^{\{0\}}]\}.
				\end{equation}
				
				Since $[\Phi]$ is inclusion monotonic according to Theorem \ref{thm1}, one has $\Phi([\mathbf{x}^{\{0\}}]) \subseteq [\Phi]([\mathbf{x}^{\{0\}}])$. Thus, it is sufficient to claim the neural network is safe if $[\Phi]([\mathbf{x}^{\{0\}}]) \cap \neg \mathcal{S} = \emptyset$ holds by Lemma \ref{lemma1}.
				
				According to the explicit expressions (\ref{thm2_1})--(\ref{thm2_5}), the computation on interval extension  $[\Phi]$ is fast. In the next step, we should discuss the conservativeness for the computation outcome of (\ref{thm2_1}).   We
				have $[\Phi]([\mathbf{x}^{\{0\}}]) = \Phi([\mathbf{x}^{\{0\}}]) + E([\mathbf{x}^{\{0\}}])$ for some interval-valued function $E([\mathbf{x}^{\{0\}}])$ with $w([\Phi]([\mathbf{x}^{\{0\}}])) = w(\Phi([\mathbf{x}^{\{0\}}])) + w(E([\mathbf{x}^{\{0\}}]))$.
				\begin{definition}
					We call
					$w(E([\mathbf{x}^{\{0\}}])) = w([\Phi]([\mathbf{x}^{\{0\}}])) - w(\Phi([\mathbf{x}^{\{0\}}])) $
					the excess width of interval extension of neural network $\Phi([\mathbf{x}^{\{0\}}])$.
				\end{definition}
				
				Explicitly, the excess width measures the conservativeness of interval extension $[\Phi]$ regarding its corresponding function $\Phi$. The following theorem gives the upper bound of the excess width $w(E([\mathbf{x}^{\{0\}}]))$. 
				
				\begin{theorem}\label{thm3}
					Consider feedforward neural network (\ref{NN}) with an interval input $[\mathbf{x}^{\{0\}}]$, the excess width $w(E([\mathbf{x}^{\{0\}}]))$ satisfies
					\begin{equation}\label{thm3_1}
					w(E([\mathbf{x}^{\{ 0\} } ])) \leq\gamma w([\mathbf{x}^{\{0\}} ])
					\end{equation}
					where $\gamma = \xi^{L}\prod\nolimits_{\ell = 1}^L {\left\| {\mathbf{W}^{\{ \ell\} } } \right\|} $.
				\end{theorem}
				\begin{proof}
					We have $[\Phi]([\mathbf{x}^{\{0\}}]) = \Phi([\mathbf{x}^{\{0\}}]) + E([\mathbf{x}^{\{0\}}])$ for some $E([\mathbf{x}^{\{0\}}])$ and 
					\begin{align*}
						w(E([\mathbf{x}^{\{0\}}])) &= w([\Phi]([\mathbf{x}^{\{0\}}])) - w(\Phi([\mathbf{x}^{\{0\}}]))
						\\
						&\leq w([\Phi ]([\mathbf{x}^{\{ 0\} } ])) 
						\\
						& \leq \xi ^L \prod\nolimits_{\ell = 1}^L {\left\| {\mathbf{W}^{\{ \ell\} } } \right\|} w([\mathbf{x}^{\{0\}} ])
					\end{align*}
					which means (\ref{thm3_1}) holds.
				\end{proof}
				
				Given a neural network $\Phi$ which means $\mathbf{W}^{\{\ell\}}$ and $\xi$ are fixed, Theorem \ref{thm3} implies that a less conservative result can be only obtained by reducing the width of input interval $[\mathbf{x}^{\{0\}}]$. On the other hand, a smaller $w([\mathbf{x}^{\{0\}}])$ means more subdivisions of an input interval which will bring more computational cost. Therefore, how to generate appropriate subdivisions of an input interval is the key for safety verification of neural networks in the framework of interval analysis. In the next section, an efficient specification-guided method is proposed to address this problem.
				
				\subsection{Specification-Guided Safety Verification}
				Inspired by the Moore-Skelboe algorithm \cite{skelboe1974computation}, we propose a specification-guided algorithm, which generates fine subdivisions particularly with respect to specification, and also avoid unnecessary subdivisions on the input interval for safety verification, see Algorithm \ref{alg1}. 
				\begin{algorithm}[ht!]
					\caption{Specification-Guided Safety Verification} \label{alg1}
					\begin{algorithmic}[1]
						\Require A feedforward neural network $\Phi:\mathbb{R}^{n^{\{0\}}} \to \mathbb{R}^{n^{\{L\}}}$, an input set $\mathcal{X} \subseteq \mathbb{R}^{n^{\{0\}}}$, a safety specification $\mathcal{S} \subseteq \mathbb{R}^{n^{\{L\}}}$, a tolerance $\varepsilon > 0$
						\Ensure Safe or Uncertain
						
						\State $\underline{x}_i \gets \min_{\mathbf{x}\in\mathcal{X}}(x_i)$, $\overline{x}_i \gets \max_{\mathbf{x}\in\mathcal{X}}(x_i)$
						\State $[\mathbf{x}] \gets [\underline{x}_1,\overline{x}_1]\times\ldots,\times[\underline{x}_{n^{\{0\}}},\overline{x}_{n^{\{0\}}}]$
						\State $[\mathbf{y}] \gets [\Phi]([\mathbf{x}])$
						\State $\mathcal{M} \gets \{([\mathbf{x}],[\mathbf{y}])\}$
						\While{$\mathcal{M} \neq \emptyset$}
						\State Select and remove an element $([\mathbf{x}],[\mathbf{y}])$ from $\mathcal{M}$
						\If{$[\mathbf{y}]\cap\neg\mathcal{S} = \emptyset$}
						\State Continue
						\Else
						\If{$w(\mathbf{[x]}) > \varepsilon$}
						\State Bisect $[\mathbf{x}]$ to obtain $[\mathbf{x}_1]$ and $[\mathbf{x}_2]$
						\For{$i=1:1:2$} 
						\If{$[\mathbf{x}_i]\cap \mathcal{X} \neq \emptyset$}
						\State $[\mathbf{y}_i] \gets [\Phi]([\mathbf{x}_i])$
						\State $\mathcal{M} \gets \mathcal{M} \cup \{([\mathbf{x}_i],[\mathbf{y}_i])\}$
						\EndIf
						\EndFor
						\Else
						\State \Return Uncertain
						\EndIf
						\EndIf
						\EndWhile
						\State \Return Safe
					\end{algorithmic}
				\end{algorithm}
				
				The implementation of the specification-guided  algorithm shown in Algorithm \ref{alg1} checks that the intersection between output set and unsafe region is empty, within a pre-defined tolerance $\varepsilon$. This is accomplished by dividing and checking the initial input interval into increasingly
				smaller sub-intervals.
				
				\begin{itemize}
					\item \textbf{Initialization.} Set a tolerance $\varepsilon>0$. Since our approach is based on interval analysis, convert input set $\mathcal{X}$ to an interval $[\mathbf{x}]$ such that $\mathcal{X} \subseteq [\mathbf{x}]$. Compute the initial output interval $[\mathbf{y}] = [\Phi]([\mathbf{x}])$. Initialize set $\mathcal{M} = \{([\mathbf{x}],[\mathbf{y}])\}$.     
					\item \textbf{Specification-guided bisection.}
					This is the key in the algorithm. Select an element $([\mathbf{x}],[\mathbf{y}])$ for specification-guided bisection. If the output interval $[\mathbf{y}]$ of sub-interval $[\mathbf{x}]$ has no intersection with the unsafe region, we can discard this sub-interval for the subsequent dividing and checking since it has been proven safe. Otherwise, the bisection action will be activated to produce finer subdivisions to be added to $\mathcal{M}$ for subsequent checking. The bisection process is guided by the given safety specification, since the activations of bisection actions are totally determined by the non-emptiness of the intersection between output interval sets and the given unsafe region. This distinguishing feature leads to finer subdivisions when the output set is getting close to the unsafe region, and on the other hand coarse subdivisions are sufficient for safety verification when the output set is far wary from the unsafe area. Therefore, unnecessary computational cost can be avoided. In the experiments section, it will be clearly observed how the bisection actions are guided by safety specification in a numeral example.
					\item \textbf{Termination.} The specification-guided bisection procedure continues until $\mathcal{M}=\emptyset$ which means all sub-intervals have been proven safe, or the width of subdivisions becomes less than the pre-defined tolerance $\varepsilon$ which leads to an uncertain conclusion for the safety. Finally, when Algorithm \ref{alg1} outputs an uncertain verification result, we can select a smaller tolerance $\varepsilon$ to perform the safety verification.
				\end{itemize}

				\section{Experiments}
				\subsection{Random Neural Network}
				To demonstrate how the specification-guided idea works in safety verification, a neural network with two inputs and two outputs is proposed. The neural network has 5 hidden layers, and each layer contains 10 neurons. The weight matrices and bias vectors are randomly generated. The input set is assumed to be $[\mathbf{x}^{\{0\}}] = [-5,5]\times [-5,5]$ and the unsafe region is $\neg \mathcal{S} = [1,\infty)\times[1,\infty)$. 
				
				\begin{table}
					\centering
					\caption{Comparison on number of intervals and computational time to existing approach}\label{tab1}
					\begin{tabular}{c|c|c}
						\hline
						& Intervals 	& Computational Time   \\
						\hline
						Algorithm \ref{alg1}   &  4095 &  21.45 s \\
						\hline
						Xiang et al. 2018 & 111556 & 294.37 s\\
						\hline
					\end{tabular}
				\end{table} 
				
				We execute Algorithm \ref{alg1} with termination parameter $\varepsilon = 0.01$, the safety can be guaranteed by partitioning $[\mathbf{x}^{\{0\}}]$ into 4095 interval sets. The specification-guided partition of the input space is shown in Figure \ref{fig1}. A non-uniform input space partition is generated based on the specification-guided scheme. An obvious specification-guided effect can be observed in Figure \ref{fig1}.  The specification-guided method requires much less computational complexity compared to the approach in \cite{xiang2018output} which utilizes a uniform partition of input space, and a comparison is listed in Table \ref{tab1}. The
				computation is carried out using Matlab 2017
				on a personal computer with Windows 7, Intel Core i5-4200U, 1.6GHz, 4 GB
				RAM. It can be seen that the number of interval sets and computational time have been significantly reduced to 3.67\% and 7.28\%, respectively, compared to those needed in~\cite{xiang2018output}. Figure \ref{fig2} illustrates the union of 4095 output interval sets, which has no intersection with the unsafe region, illustrating the safety specification is verified. Figure \ref{fig2} shows that the output interval estimation is guided to be tight when it comes close to unsafe region, and when it is far way from the unsafe area, a coarse estimation is sufficient to verify safety.

				\begin{figure}[ht!]
					\includegraphics[width=9cm]{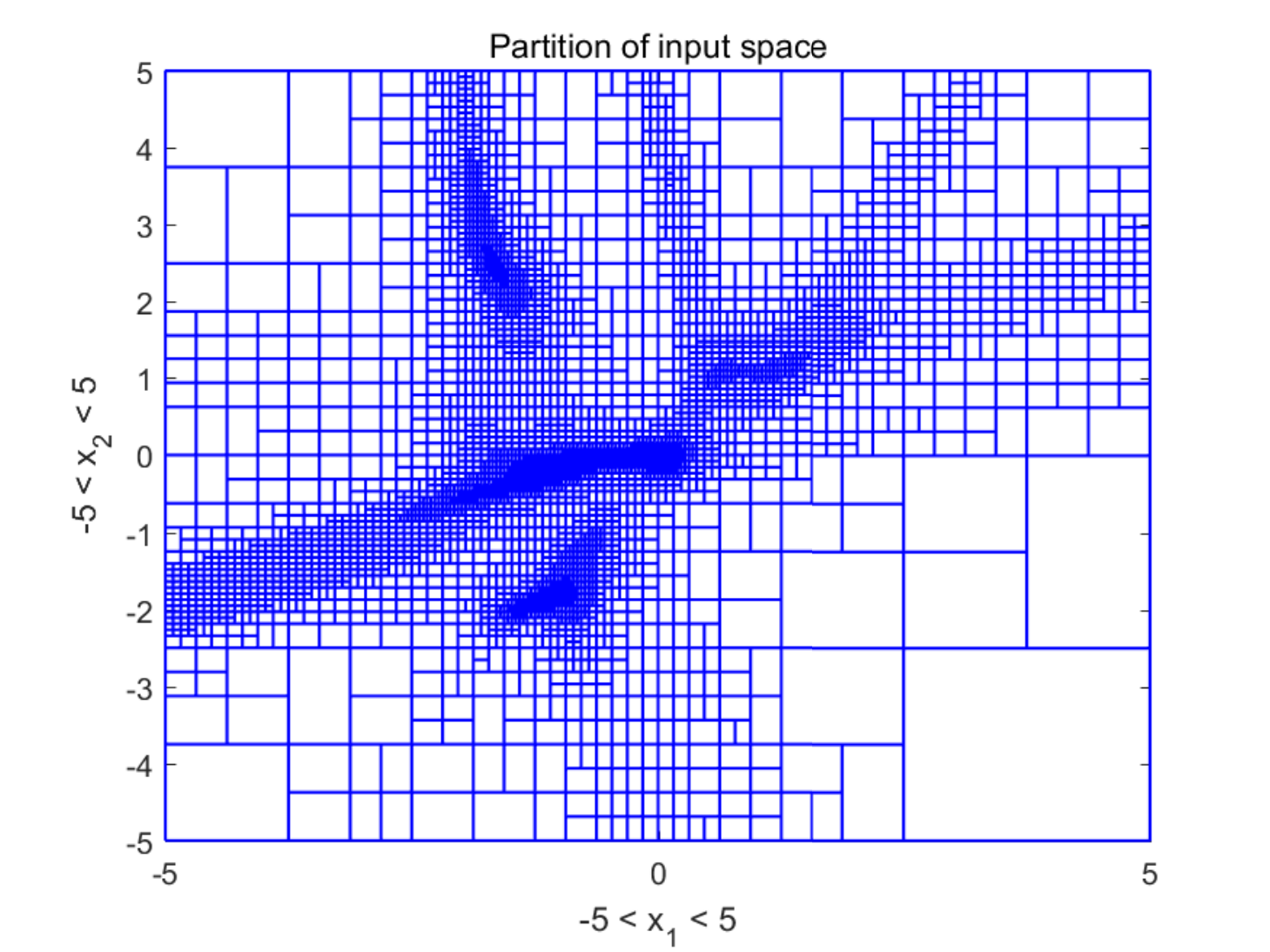}
					\caption{Specification-guided bisections of input interval by Algorithm \ref{alg1}. Guided by safety specification, finer partitions are generated when the output intervals are close to the unsafe region, and coarse partitions are generated when the output intervals are far wary. }
					\label{fig1}
				\end{figure}
				\begin{figure}
					\includegraphics[width=9cm]{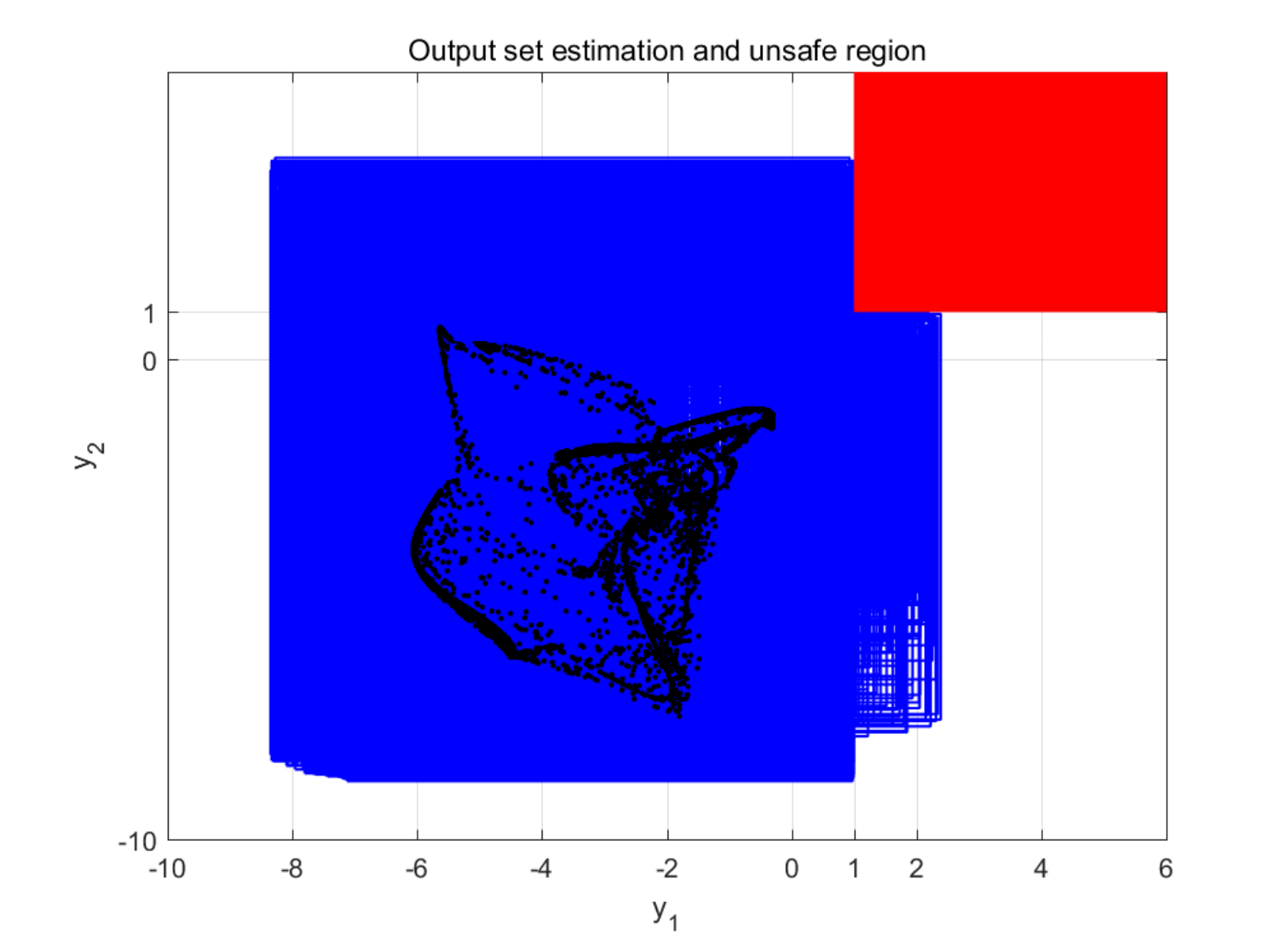}
					\caption{Output set estimation of neural networks. Blue boxes are output intervals, red area is unsafe region, black dots are 5000 random outputs.}
					\label{fig2}
				\end{figure}

				\subsection{Robotic Arm Model}
				\begin{figure}[ht!]
					\begin{center}
						\includegraphics[width=4cm]{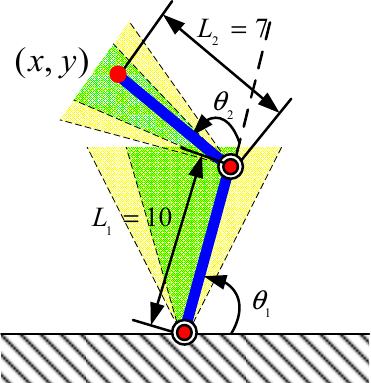}
						\caption{Robotic arm with two joints. The normal working zone of $(\theta_1,\theta_2)$ is colored in green $\theta_1,\theta_2 \in [\frac{5\pi}{12},\frac{7\pi}{12}]$. The buffering zone is in yellow $\theta_1,\theta_2 \in [\frac{\pi}{3},\frac{5\pi}{12}] \cup [\frac{7\pi}{12},\frac{2\pi}{3}] $. The forbidden zone is $\theta_1,\theta_2 \in [0,\frac{\pi}{3}] \cup [\frac{2\pi}{3},2\pi] $.
						}
						\label{robotic_arm}
					\end{center}
				\end{figure}
				
				In \cite{xiang2018output}, a  \emph{learning forward kinematics} of a robotic arm model with two joints is proposed, shown in Figure \ref{robotic_arm}. 
				The learning task is using a feedforward neural network to predict the position $(x,y)$
				of the end with knowing the joint angles $(\theta_1,\theta_2)$. The input space $[0,2\pi]\times [0,2\pi]$ for $(\theta_1,\theta_2)$  is classified into three zones for its operations: normal working zone $\theta_1,\theta_2 \in [\frac{5\pi}{12},\frac{7\pi}{12}]$, buffering zone $\theta_1,\theta_2 \in [\frac{\pi}{3},\frac{5\pi}{12}] \cup [\frac{7\pi}{12},\frac{2\pi}{3}] $ and forbidden zone $\theta_1,\theta_2 \in [0,\frac{\pi}{3}] \cup [\frac{2\pi}{3},2\pi]$. The detailed formulation for this robotic arm model and neural network training can be found in \cite{xiang2018output}. 
				
				The safety specification for the position $(x,y)$ is 
				$\mathcal{S}=\{(x,y)\mid -14 \le x\le 3~\mathrm{and}~1 \le y \le 17\}$. The input set of the robotic arm is the union of normal working  and buffering zones, that is $(\theta_1,\theta_2) \in [\frac{\pi}{3},\frac{2\pi}{3}] \times [\frac{\pi}{3},\frac{2\pi}{3}]$.
				In the safety point of view, the neural network needs to be verified that all the outputs produced by the inputs in the normal working zone and buffering zone will satisfy  safety specification $\mathcal{S}$. In \cite{xiang2018output}, a uniform partition for input space is used, and thus 729 intervals are produced to verify the safety property. Using our specification-guided approach, the safety can be guaranteed by partitioning the input space into only 15 intervals, see Figure \ref{fig3} and Figure \ref{fig4}. Due to the small number of intervals involved in the verification process, the computational time is only 0.27 seconds for specification-guided approach.

				\begin{figure}[ht!]
					\includegraphics[width=9cm]{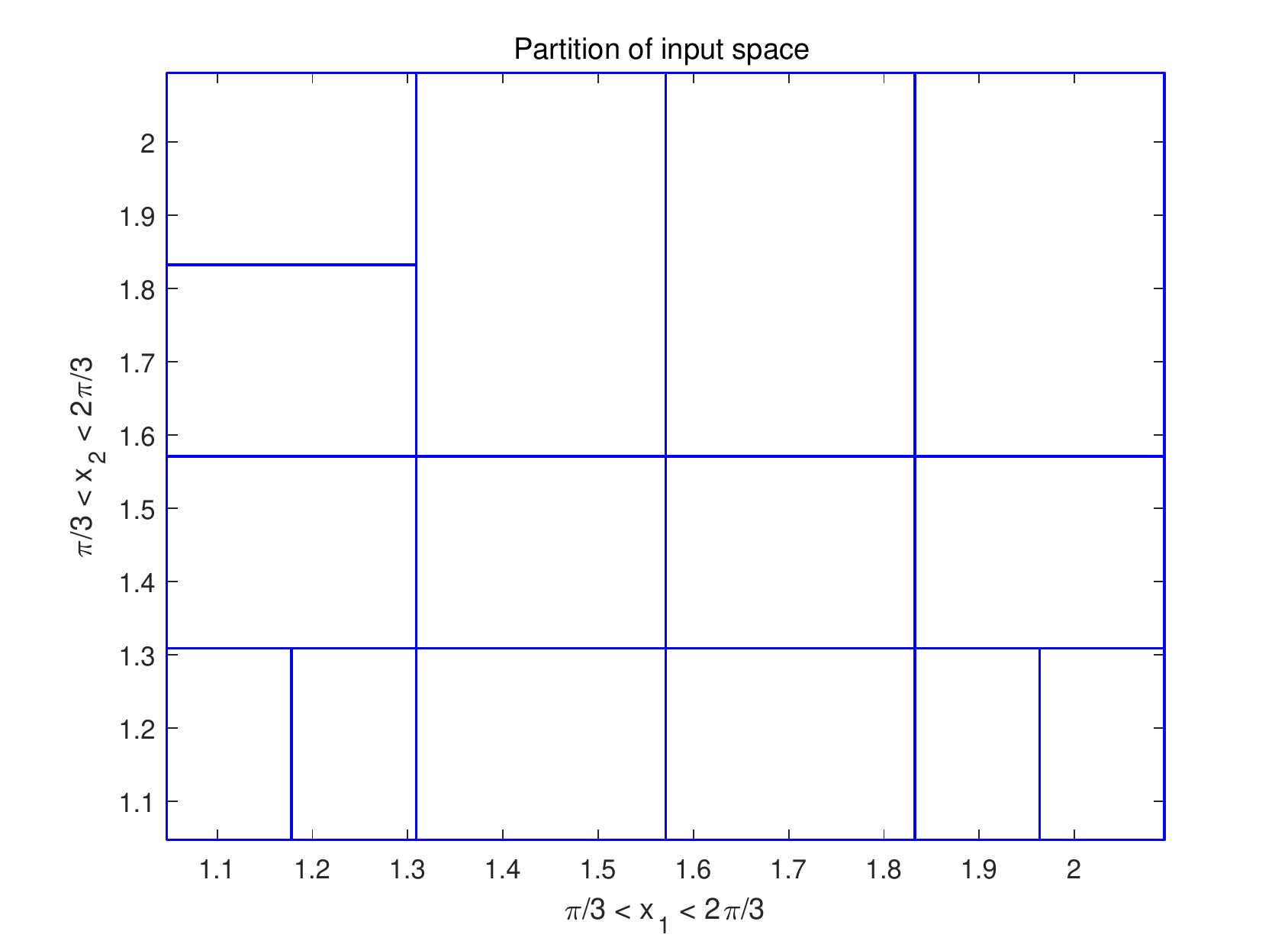}
					\caption{15 sub-intervals for robotic arm safety verification.}
					\label{fig3}
				\end{figure}
				
				\begin{figure}
					\includegraphics[width=9cm]{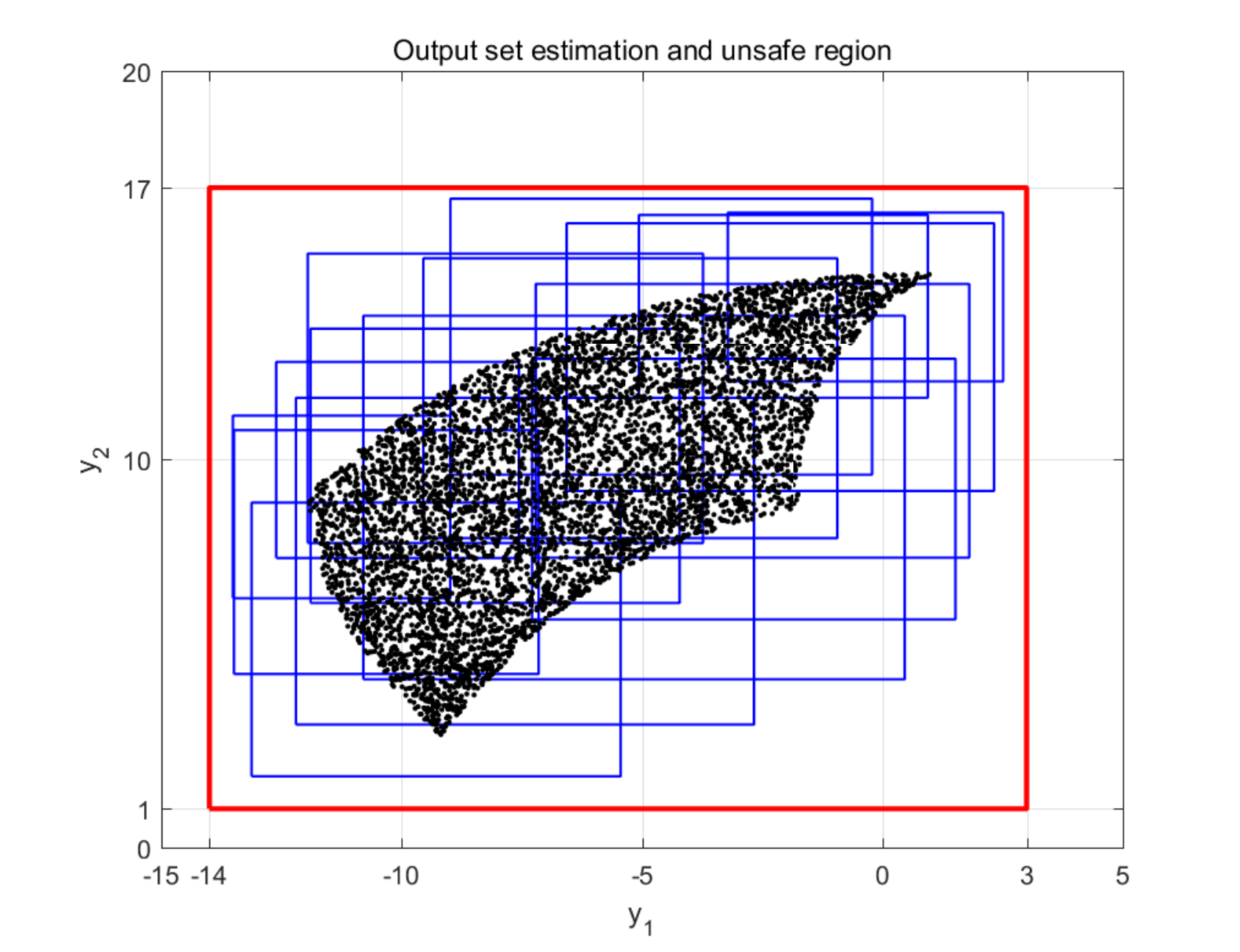}
					\caption{Safety verification for neural network of robotic arm.   Blue boxes are output intervals, red box are boundary for unsafe region, black dots are 5000 random outputs. 15 output intervals are sufficient to prove the safety. }
					\label{fig4}
				\end{figure}

				\begin{figure}[ht!]
					\includegraphics[width=9cm]{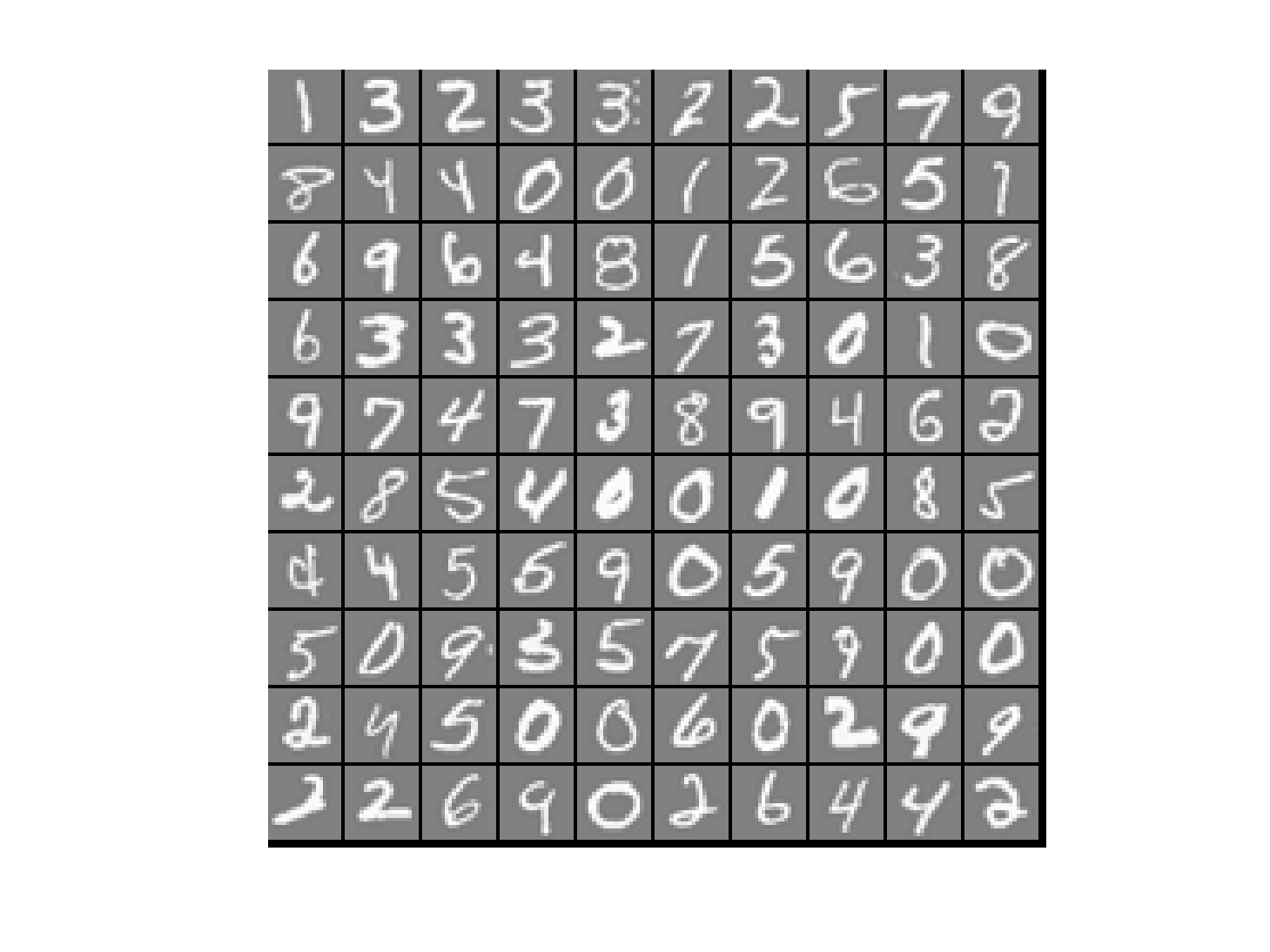}
					\caption{Examples from the MNIST handwritten digit dataset. }
					\label{fig5}
				\end{figure}
				\begin{figure}
					\centering
					\includegraphics[width=9cm]{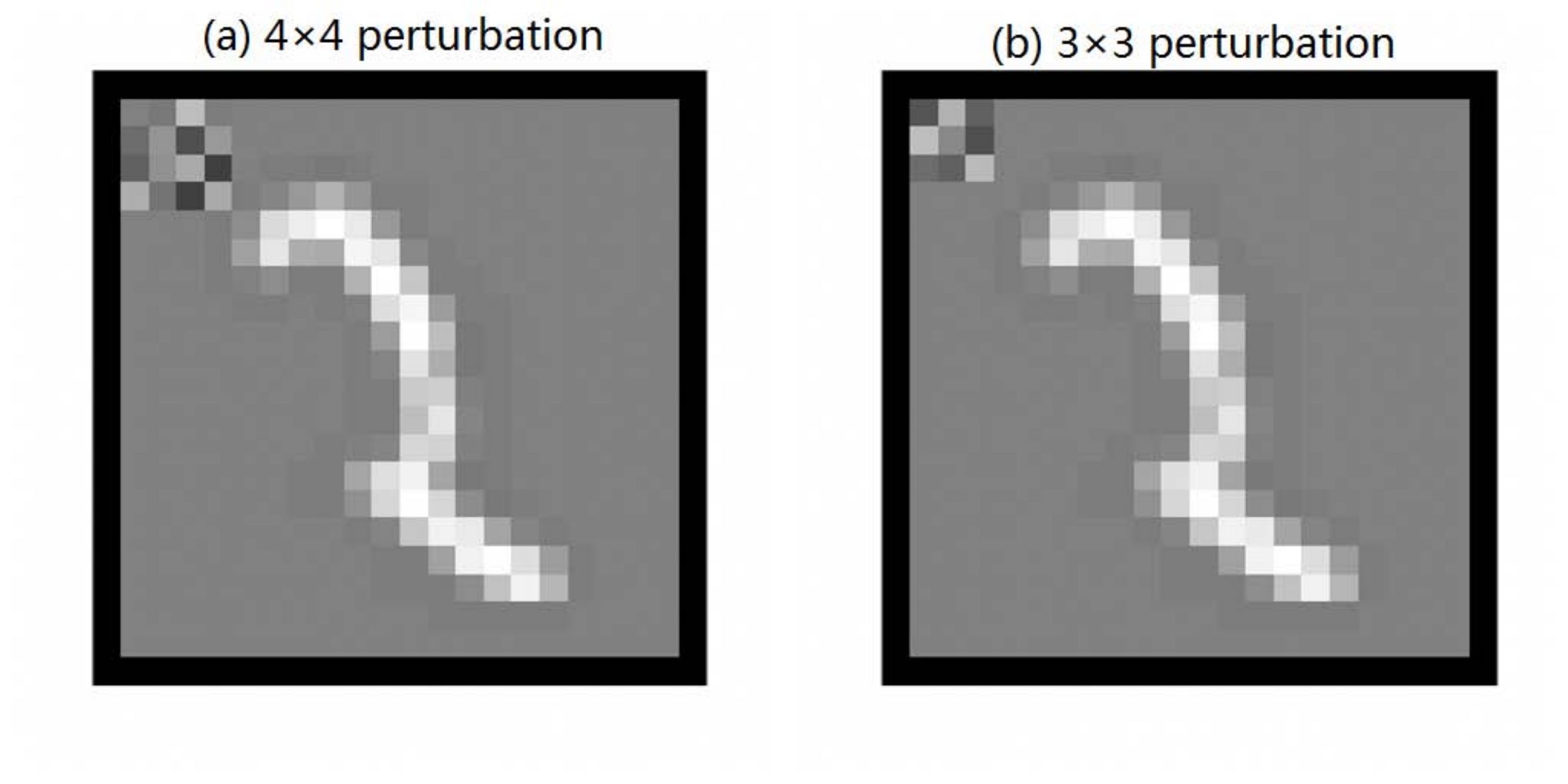}
					\caption{Perturbed image of digit 2 with perturbation in $[-0.5.0.5]$. (a) $4 \times 4$ perturbation at the left-top corner, the neural network will wrongly label it as digit 1. (b) $3 \times 3$ perturbation at the left-top corner, the neural network can be proved to be robust for this class of perturbations. }
					\label{fig6}
				\end{figure}

				\subsection{Handwriting Image Recognition}
				In this handwriting image recognition task, we use 5000 training examples of handwritten digits which is a subset of the MNIST handwritten digit dataset (http://yann.lecun.com/exdb/mnist/), examples from the dataset are shown in Figure \ref{fig5}. Each training
				example is a 20 pixel by 20 pixel grayscale image of the digit. Each pixel is
				represented by a floating point number indicating the grayscale intensity at that location. We first train a neural network with 400 inputs, one hidden layer with 25 neurons and 10 output units corresponding to the 10 digits. The activation functions for both hidden and output layers are sigmoid functions. A trained neural network with about 97.5\% accuracy is obtained.
				
				Under adversarial perturbations, the neural network may produce a wrong prediction. For example in Figure \ref{fig6}(a) which is an image of digit $2$, the label predicted by the neural network will turn to $1$ as a $4 \times 4$ perturbation belonging to $[-0.5,0.5]$ attacks the left-top corner of the image. With our developed verification method, we wish to prove that the neural network is robust to certain classes of perturbations, that is no perturbation belonging to those classes can alter the prediction of the neural network for a perturbed image. Since there exists one adversarial example for $4 \time 4$ perturbations at the left-top corner, it implies this image is not robust to this class of perturbation. We consider another class of perturbations, $3\times 3$ perturbations at the left-top corner, see Figure \ref{fig6}(b). Using Algorithm \ref{alg1},  the neural network can be proved to be robust to all $3\times 3$ perturbations located at at the left-top corner of the image, after 512 bisections.
				
				Moreover, applying Algorithm \ref{alg1} to all 5000 images with $3\times 3$ perturbations belonging to $[-0.5,0.5]$ and located at the left-top corner, it can be verified  that the neural network is robust to this class of perturbations for all images. This result means this class of perturbations will not affect the prediction accuracy of the neural network. The neural network is able to maintain its 97.5\% accuracy even subject to any perturbations belonging to this class of $3 \times 3$ perturbations.

				\section{Conclusion and Future Work}
				In this paper, we introduce a specification-guided approach for safety verification of feedforward neural networks with general activation functions. By formulating the safety verification problem into the framework of interval analysis, a fast computation formula for calculating output intervals of feedforward neural networks is developed. Then, a safety verification algorithm which is called specification-guided is developed. The algorithm is specification-guided since the activation of bisection actions are totally determined by the existence of intersections between the computed output intervals and unsafe sets. This distinguishing feature makes the specification-guided approach be able to avoid unnecessary computations and significantly reduce the computational cost. Several experiments are proposed to show the advantages of our approach. 
				
				Though our approach is general in the sense that it is not tailored to specific activation functions, the specification-guided idea has potential to be further applied to other methods dealing with specific activation functions such as ReLU neural networks to enhance their scalability. Moreover, since our approach can compute the output intervals of a neural network, it can be incorporated with other reachable set estimation methods to compute the dynamical system models with neural network components inside such as extension of \cite{xiang2018output} to closed-loop systems \cite{xiang2019reachable} and neural network models of nonlinear dynamics \cite{xiang2018reachable_b}. 

\bibliographystyle{ieeetr}
\bibliography{ref}

\end{document}